\documentclass[letterpaper, 10 pt, conference]{ieeeconf}
\IEEEoverridecommandlockouts
\usepackage{times}

\let\labelindent\relax

\makeatletter
\let\NAT@parse\undefined
\makeatother

\usepackage{multicol}
\usepackage[bookmarks=true]{hyperref}

	\usepackage{cite}
	\usepackage{amsmath}
	\usepackage{amsthm}
	\usepackage{amsfonts}
	\usepackage{amssymb}
	\usepackage{ragged2e}
	\usepackage{graphicx}
	\usepackage{cancel}
	\usepackage{mathtools}
	\usepackage{tabularx}
	\usepackage{arydshln}
	\usepackage{tensor}
	\usepackage{array}
	\usepackage[dvipsnames]{xcolor}
	\usepackage[boxed]{algorithm}
	\usepackage[noend]{algpseudocode}
	\usepackage{listings}
	\usepackage{textcomp}
	\usepackage{mathrsfs}
	\usepackage{bbm}
	\usepackage{tikz}
	\usepackage{tikz-cd}
	\usepackage{enumitem}
	\usepackage{arydshln}
	\usepackage{relsize}
	\usepackage{multirow}
	\usepackage{upgreek}
	\usepackage{ifthen}
	\usepackage{yhmath}
	\usepackage{blkarray}
	\usepackage{dashrule}
	\usepackage{subcaption}
	\usepackage{overpic}
    \usepackage{lipsum}

		\newcommand{\inv}{^{-1}}

		\newcommand{\of}{\circ}

		\newcommand{\restr}[1]{\left.#1\right|}

		\newcommand{\card}[1]{\left|#1\right|}
		\newcommand{\set}[1]{\left\{#1\right\}}
		
		\newcommand{\naturals}{\mathbb{N}}
		\newcommand{\N}{\naturals}

		\newcommand{\reals}{\mathbb{R}}
		\newcommand{\R}{\reals}

		\newcommand{\cut}{\setminus}

		\newcommand{\iprod}[1]{\left<#1\right>}

		\newcommand{\paren}[1]{\left(#1\right)}

		\newcommand{\ptxt}[1]{\textrm{\textnormal{#1}}}
		
		\newcommand{\mc}[1]{\mathcal{#1}}
		\newcommand{\ms}[1]{\mathscr{#1}}
		\newcommand{\mb}[1]{\mathbb{#1}}
		\newtheorem{theorem}{Theorem}

		\newtheorem{lemma}{Lemma}
		\newtheorem{corollary}{Corollary}

	\usepackage{letltxmacro}
	\LetLtxMacro\orgvdots\vdots
	\LetLtxMacro\orgddots\ddots

	\makeatletter
	\DeclareRobustCommand\vdots{%
		\mathpalette\@vdots{}%
	}
	\newcommand*{\@vdots}[2]{%
		\sbox0{$#1\cdotp\cdotp\cdotp\m@th$}%
		\sbox2{$#1.\m@th$}%
		\vbox{%
			\dimen@=\wd0 %
			\advance\dimen@ -3\ht2 %
			\kern.5\dimen@
			\dimen@=\wd2 %
			\advance\dimen@ -\ht2 %
			\dimen2=\wd0 %
			\advance\dimen2 -\dimen@
			\vbox to \dimen2{%
				\offinterlineskip
				\copy2 \vfill\copy2 \vfill\copy2 %
			}%
		}%
	}
	\DeclareRobustCommand\ddots{%
		\mathinner{%
			\mathpalette\@ddots{}%
			\mkern\thinmuskip
		}%
	}
	\newcommand*{\@ddots}[2]{%
		\sbox0{$#1\cdotp\cdotp\cdotp\m@th$}%
		\sbox2{$#1.\m@th$}%
		\vbox{%
			\dimen@=\wd0 %
			\advance\dimen@ -3\ht2 %
			\kern.5\dimen@
			\dimen@=\wd2 %
			\advance\dimen@ -\ht2 %
			\dimen2=\wd0 %
			\advance\dimen2 -\dimen@
			\vbox to \dimen2{%
				\offinterlineskip
				\hbox{$#1\mathpunct{.}\m@th$}%
				\vfill
				\hbox{$#1\mathpunct{\kern\wd2}\mathpunct{.}\m@th$}%
				\vfill
				\hbox{$#1\mathpunct{\kern\wd2}\mathpunct{\kern\wd2}\mathpunct{.}\m@th$}%
			}%
		}%
	}
	\makeatother

	\tikzset{
	  symbol/.style={
		draw=none,
		every to/.append style={
		  edge node={node [sloped, allow upside down, auto=false]{$#1$}}}
	  }
	}

\usepackage{cleveref}
\Crefname{figure}{Fig.}{Figs.}
\Crefname{equation}{Eq.}{Eqs.}
\Crefname{lemma}{Lemma}{Lemmata}
\Crefname{proposition}{Proposition}{Propositions}
\Crefname{assumption}{Assumption}{Assumptions}
\Crefname{theorem}{Theorem}{Theorems}
\Crefname{section}{Section}{Sections}
\Crefname{subsection}{Subsection}{Subsections}
\Crefname{appendix}{Appendix}{Appendices}
\Crefname{corollary}{Corollary}{Corollaries}

\newcommand{\rmd}{\mathrm{d}}

\DeclareMathOperator{\SO}{SO}
\DeclareMathOperator{\SE}{SE}

\newcommand{\mrm}[1]{\mathrm{#1}}

\newcommand{\free}{\mrm{free}}

\DeclareMathOperator{\Vol}{Vol}
\newcommand{\eps}{\epsilon}
\newcommand{\Prob}{\mb{P}}
\newcommand{\E}{\mb{E}}

\usepackage{color-edits}
\addauthor{tc}{cyan}
\addauthor{rt}{red}
\addauthor{ss}{orange}

\title{\LARGE \bf
Sampling-Based Motion Planning with Discrete Configuration-Space Symmetries
}

\author{
Thomas Cohn$^{1}$, Russ Tedrake$^{1}$
\thanks{
$^{1}$The authors are with the Computer Science and Artificial Intelligence Laboratory (CSAIL), Massachusetts Institute of Technology, 32 Vassar St, Cambridge, MA, 02139 {\tt\small [tcohn|russt]@mit.edu}}%
}

\begin{document}

\maketitle
\thispagestyle{empty}
\pagestyle{empty}

\begin{abstract}%
    When planning motions in a configuration space that has underlying symmetries (e.g. when manipulating one or multiple symmetric objects), the ideal planning algorithm should take advantage of those symmetries to produce shorter trajectories.
However, finite symmetries lead to complicated changes to the underlying topology of configuration space, preventing the use of standard algorithms.
We demonstrate how the key primitives used for sampling-based planning can be efficiently implemented in spaces with finite symmetries.
A rigorous theoretical analysis, building upon a study of the geometry of the configuration space, shows improvements in the sample complexity of several standard algorithms.
Furthermore, a comprehensive slate of experiments demonstrates the practical improvements in both path length and runtime.

\end{abstract}

\section{Introduction}
\label{sec:introduction}
When solving motion planning problems that have underlying symmetries, the ideal planner should leverage this additional structure to find lower cost paths, reduce runtime, or both.
In this paper, we study discrete symmetries, where for each system state, there is a finite set of other states that are functionally indistinguishable.
Such symmetries appear in contact-rich manipulation problems when the object (or objects) being manipulated is a symmetric object, such as a cube~\cite{dafle2014extrinsic,chavan2018hand,pang2023global,suh2025ctr}.
Outside of robotic manipulation, these symmetries appear in assembly/disassembly planning tasks, where paths must be planned for individual (possibly symmetric) parts~\cite{ghandi2015review}.
Finally, when using unit quaternions to represent orientations in 3D space, the equivalence of two antipodal unit quaternions can be interpreted as a form of symmetry.

A common property of all these symmetries is that they appear explicitly in the system's \emph{configuration space}.
In configuration space (or C-space), each coordinate corresponds to an individual degree of freedom, so a single point describes the position and orientation of each part of the robot~\cite{lozano1990spatial}.
For many robots, C-space can be interpreted as a bounded subset of Euclidean space.
(For example, in the case of a robot arm, where each joint has physical limits to its motion.)
But this is not always the case, and the topology of C-space may be inherently non-Euclidean.
For example, the C-space of a revolute joint without limits is the unit circle $S^1$, and the C-space of a mobile robot is the Lie group $\SE(2)$ (or $\SE(3)$ if the robot can move in 3D space).
Finally, closed kinematic chains (or the presence of certain constraints on the motion planning problem) will cause the set of valid configurations to be a curved, measure-zero subset of the full ambient C-space.
In all of these cases, the standard strategy is to represent C-space as a smooth manifold~\cite{berenson2009manipulation,biggs2018motion,kingston2018sampling}. 

A common attribute of the standard types of manifolds considered in robotics is that they have clear descriptions.
Continuous revolute joints can be represented using straightforward switching logic for the wraparound.
There are several well-understood representations available for $\SO(3)$ and $\SE(3)$~\cite{geist2024learning}.
And kinematically-constrained spaces are still a subset of an ambient C-space~\cite{kingston2018sampling}.

However, representing the C-space of an object with symmetries is more challenging.
Mathematically, it is described as a \emph{quotient manifold}, a topological construction that does not yield an explicit representation as a subset of Euclidean space~\cite[p. 540]{lee2012smooth}.
While \emph{continuous} symmetries can be handled via dropping dimensions (e.g., the C-space of a sphere is just its position in $\R^3$), in the case of discrete symmetries, the dimension of the resulting space is unchanged, so this strategy clearly cannot work.

In this paper, we present a strategy for extending sampling-based motion planning algorithms to exploit discrete C-space symmetries.
We show how the key primitive operations needed for sampling-based planning (distance computation, local and global sampling, and local planning) can all be implemented in closed-form for the quotient manifolds of interest.
We demonstrate how this can be viewed through the lens of the IMACS (Implicit Manifold Conﬁguration Space) framework for planning on general manifolds~\cite{kingston2019exploring}, even without an explicit embedding into Euclidean space.
We present a rigorous theoretical analysis of several standard sampling-based planning algorithms when planning in the quotient spaces of interest, quantifying the improvement in bounds on sample complexity.
We also present a comprehensive slate of experiments to demonstrate the efficacy of our strategies for symmetry-aware planning, including path planning in both 2D and 3D for a variety of objects with symmetries.
Across our experiments, leveraging the symmetries yields shorter paths, and despite the more complex approach, actually leads to shorter runtimes.

\section{Related Work}
\label{sec:related_work}
Motion planning is one of the most-studied problems in robotics.
Algorithms for solving this problem can broadly be divided into two categories: \emph{trajectory optimizers} and \emph{sampling-based planners}.
In a trajectory optimization problem, the control points of the robot's trajectory are used as decision variables, and costs and constraints are imposed to shape the desired trajectories~\cite{zucker2013chomp,kalakrishnan2011stomp,toussaint2014komo,howell2019altro}.
These optimization problems are almost always nonconvex, so trajectory optimizers frequently get stuck in local minima, and without a good initialization, may fail to compute a feasible trajectory altogether~\cite[Ch.6]{manipulation}.
To avoid local minima, one must try many different random initializations~\cite{sundaralingam2023curobo} or build complex representations of the collision-free space~\cite{marcucci2023motion}.

Sampling-based planners function by drawing random samples from C-space, rejecting those that are in collision, and connecting samples if the path between them is collision free.
The \emph{Rapidly-Exploring Random Tree} (RRT)~\cite{lavalle1998rapidly} and \emph{Probabilistic Roadmap} (PRM)~\cite{kavraki1996probabilistic} algorithms are the most widely-used sampling-based planning algorithms, although there are many others~\cite{orthey2023sampling}.
RRT incrementally constructs a space-filling tree, rooted at the start configuration, until a path to the goal configuration is found.
PRM builds a roadmap in an offline construction phase, so it only has to link a start and goal configuration into the roadmap to find a path.

The idea of exploiting symmetry to improve planning is not a new idea.
System symmetries can be exploited for model predictive control~\cite{teng2022error}.
In planning-as-search, graph-theoretic symmetries can be exploited for dramatic improvements~\cite{fox2002extending,domshlak2013symmetry}.
Symmetries in certain motion primitives for nonlinear systems can also be leveraged~\cite{pedrosa2021graph,frazzoli2005maneuver}.
Cheng et al.~\cite{cheng2003exploiting} used system symmetries to avoid costly numerical integration steps in kinodynamic planning.
However, there is not much motion planning literature specifically focusing on symmetries in configuration space.
Orthey et al.~\cite{orthey2024multilevel} examined planning on quotient manifolds, but each quotient had to be positive dimensional, and this approach could not handle more topologically-complex quotient spaces.

The symmetries we consider in this paper have finitely many representatives, so a straightforward way to apply existing work is to cast the problem as a continuous planning problem with discrete decisions describing the representative of the planning goal.
This can be interpreted as task-and-motion planning~\cite{garrett2021integrated}, or one can avoid explicitly modeling the decision with a cost function that captures the symmetries.
Scott et al.~\cite{scott2018trajectory} used this strategy for differential dynamic programming, imposing a cost on the final state that was (equally) minimized at any of the goal representatives.

Another possible strategy would be to embed the quotient manifold as a submanifold of Euclidean space, so that one could apply the rich literature studying that problem~\cite{kingston2018sampling,kingston2019exploring,bordalba2022direct}.
The Whitney~\cite[p.134]{lee2012smooth} and Nash~\cite{nash1956imbedding} embedding theorems guarantee the existence of such embeddings, but are nonconstructive.
Recent papers from the mathematical statistics community have constructed explicit embeddings of the quotient manifolds arising from discrete symmetries of $\SO(3)$~\cite{arnold2018statistics,hielscher2021locally}, but the lack of a closed-form inverse mapping makes it unclear how to enforce collision-avoidance.

Perhaps the greatest success for leveraging symmetries has been the world of equivariant machine learning, where symmetry-aware learning algorithms can achieve significant improvements in sample complexity, compared to their symmetry-unaware counterparts.
These sample complexity gains have been leveraged in grasping~\cite{zhu2023robot}, simultaneous localization and mapping~\cite{van2019geometric}, and point cloud registration~\cite{zhu2022correspondence}.
Existing work has also considered the applicability of symmetries in reinforcement and imitation learning for robotics~\cite{zhao2024equivariant}, but modifying these strategies
to handle object-level equivariances is challenging.
Yang et al.~\cite{yang2024equibot} used segmentation to find object-centric point clouds, enabling object-level equivariances, but this approach can only work for a single symmetric object.

\section{Background}
\label{sec:background}
Our approach for leveraging symmetries is based on properties of the geometry of the configuration space, which we study through the lens of group theory and differential geometry.
For more details, we recommend Artin~\cite{artin2010algebra} for group theory and Lee~\cite{lee2012smooth,lee2018introduction} for differential geometry.

\subsection{Group Theory}
\label{sec:background:group_theory}

A \emph{group} is a set $G$ with an associative \emph{group law} $\cdot:G\times G\to G$.
(We often write $gh$ to indicate $g\cdot h$ (for $g,h\in G$), omitting the symbol.)
There must be an identity element $e\in G$ and every $g\in G$ must have an inverse.
We use $|G|$ to denote the cardinality, or \emph{order}, of $G$.
A \emph{subgroup} is a subset of a group that is itself a group.
A \emph{group action} (on a set $X$) is a binary operation $\alpha:G\times X\to X$ that is \emph{compatible} with the group law.
We say that $G$ \emph{acts} on $X$ by $\alpha$, and often write $\alpha(g,x)$ as $g\cdot x$ when the action is unambiguous.
Note that every group acts on itself via the group law.
The \emph{orbit} of $x\in X$ is the set of points reachable via $G$, denoted $G\cdot x=\set{g\cdot x:g\in G}$.

We can use groups to describe the symmetries of a rigid body in 2D or 3D.
In particular, the configuration of an object in 2D is described by the group $\SE(2)$, and the subgroup $\SO(2)\subseteq\SE(2)$ is the set of orientations.
We can describe the object's symmetries as a subgroup of $\SO(2)$ that acts on $\SE(2)$ (respectively, $\SO(3)$ and $\SE(3)$ for an object in 3D).
Under this action, the orbit of a configuration of the object is the set of all other configurations that are the same under the symmetry.

\begin{table}
    \centering
    \begin{tabular}{|>{\centering\arraybackslash}p{2.2cm}|c|>{\centering\arraybackslash}p{4.2cm}|} \hline
        Group Name & Order & Examples of Corresponding Objects\\ \hline
        Cyclic Group of Order $n$ & $n$ & Pyramid with a regular $n$-gon base (abbreviated as $n$-Pyramid) \\ \hline
        Dihedral Group of Order $2n$ & $2n$ & Prism with a regular $n$-gon base (abbreviated as $n$-Prism) \\ \hline
        Alternating Group on 4 Elements & $12$ & Tetrahedron\\ \hline
        Symmetric Group on 4 Elements & $24$ & Cube, Octahedron\\ \hline
        Alternating Group on 5 Elements & $60$ & Dodecahedron, Icosahedron\\ \hline
    \end{tabular}
    \caption{
        Possible 3D symmetries as subgroups of $\SO(3)$.
        \vspace{-\baselineskip}
    }
    \label{tab:so3_subgroups}
\end{table}

The finite subgroups of $\SO(2)$ and $\SO(3)$ have been completely classified~\cite{klein2003lectures}.
In $\SO(2)$, we only need to consider $C_n$ for $n\in\N$, the cyclic group of order $n$.
This corresponds to the symmetries experienced by a regular $n$-gon.
($C_1$ corresponds to an object with no symmetries, and $C_2$ corresponds to a rectangle.)
$\SO(3)$ is more complicated; its subgroups are listed in \Cref{tab:so3_subgroups}.
There are additional common symmetries that are encompassed by our framework.
The set of symmetries experienced by a cylinder is the product of an infinite group (describing rotation about the axis of symmetry), and $C_2$ (corresponding to ``flipping'' the cylinder).
Also, if we use the unit quaternions to represent $\SO(3)$, the fact that $q$ and $-q$ represent the same orientation can be seen as the symmetry group $C_2$.
Finally, this framework can naturally represent the symmetries of a system with multiple symmetric objects.
For example, given two symmetry groups $G_1$ and $G_2$ for two objects, the symmetry group of their joint configuration space is naturally the \emph{product} group $G_1\times G_2$.

\subsection{Symmetry Groups Acting on Manifolds}
\label{sec:background:symmetry_groups}

A manifold is a locally Euclidean topological space.
We interpret a configuration space of interest as a \emph{Riemannian manifold} $(\mc Q, g)$, where $g$, the \emph{Riemannian metric}, is an inner product on the tangent space at each point of the manifold $\iprod{\cdot,\cdot}_g$, and it lets us measure the length of a curve $\gamma:[0,1]\to\mc Q$ via the \emph{length functional}
\begin{equation}
    \ms L(\gamma)=\int_0^1\sqrt{\iprod{\dot\gamma(t),\dot\gamma(t)}_g}\,\rmd t.
    \label{eq:length_functional}
\end{equation}
A curve $\gamma$ which is locally length minimizing is called a \emph{geodesic}, and it is uniquely defined by its initial position and velocity.
Geodesics need not be unique, but every geodesic of length less than the \emph{injectivity radius} $r_\mrm{inj}(\mc Q)$ is globally length-minimizing and the unique shortest path connecting its endpoints.
A function between manifolds that preserves the Riemannian metric is a \emph{local isometry}, and it preserves the length of curves under composition.
Thus, local isometries map geodesics to geodesics.

A \emph{Lie group} $G$ is a smooth manifold that is also a group, where $(g,h)\mapsto g\inv h$ is a smooth function.
(Examples include $\SO(2)$ and $\SO(3)$.)
Suppose we have a Lie group $G$ acting on a manifold $\mc M$.
We can define an equivalence relation by $x\sim y$ if $\exists g\in G$ such that $x=g\cdot y$, so the equivalence classes $[x]$ are the orbits $G\cdot x$, and we have the \emph{canonical projection} $\pi:x\mapsto[x]$.
It is natural to describe the configuration space of an object with symmetries in this way, and if certain properties are satisfied, the space itself will be highly structured.
Notably, all symmetry groups considered in \Cref{sec:background:group_theory} satisfy conditions\footnote{
    Free, proper, properly discontinuous, by isometries, and transitive.
} ensuring several key properties:
$\mc M/G=\set{[x]:x\in\mc M}$ is a smooth manifold with $\dim \mc M/G=\dim\mc M$, and geodesics of $\mc M/G$ can be lifted to geodesics of $\mc M$, so the Riemannian distance function takes the convenient form
\begin{equation}
    \rmd_{\mc M/G}([x],[y])=\inf\set{\rmd_\mc M(\tilde x,\tilde y):\tilde y\in[y]},
    \label{eq:simpler_quotient_distance}
\end{equation}
where $\tilde x\in[x]$ is an arbitrary fixed representative.
Also, with multiple symmetric objects, we can compute the distance function on a per-object basis.
So the cost to compute the distance grows linearly with the number of objects, even though the size of the symmetry group grows exponentially.

\section{Methodology}
\label{sec:methodology}
We begin by describing our strategy for planning in the quotient configuration spaces of interest, before undertaking a theoretical analysis of the performance of several common sampling-based planners within this framework.

\subsection{Sampling-Based Planning}
\label{sec:methodology:sampling_based_planning}

Consider a configuration space $\mc Q$ with collision-free subset $\mc Q_\free$.
(We also require $\mc Q_\free$ be Lebesgue-measurable.)
Given initial and goal configurations $q_0,q_1\in\mc Q_\free$, the motion planning problem is
\begin{equation}
    \begin{array}{rl}
        \operatorname{find} & \gamma:[0,1]\to\mc Q\\
        \ptxt{such that} & \gamma(0)=q_0,\;\gamma(1)=q_1,\\
        & \gamma(t)\in\mc Q_\free,\;\forall t\in[0,1].
    \end{array}
    \label{eq:regular_motion_planning_problem}
\end{equation}
There are a plethora of sampling-based planning algorithms that try to solve a variant of this problem.
For example, a \emph{multi-query} planner will use offline compute time to build up a model, which can be used to quickly compute paths between many distinct start/goal pairs.
And while \eqref{eq:regular_motion_planning_problem} simply asks for a satisficing plan, \emph{asymptotically optimal} planners will attempt to find a path of minimal length.

The IMACS framework~\cite{kingston2019exploring} demonstrated that many existing sampling-based planning algorithms on manifolds can be abstracted as higher-level approaches that only interact with the underlying geometry via a set of common primitives.
\begin{enumerate}
    \item Distance metric: a function $\rmd:\mc Q\times\mc Q\to\R_{\ge 0}$ that describes the proximity of different configurations.
    \item Global sampler: a procedure for drawing samples on all of $\mc Q$ which must almost-surely sample any positive-volume subset of $\mc Q$. (A uniform sampler is ideal.)
    \item Local sampler: a procedure for drawing samples from a ball around a given point.
    \item Local planner: a (usually deterministic) planner which must try to produce a feasible plan between two nearby configurations.
    Note that the local planner need not be \emph{complete}; it can return infeasible even if a path exists.
\end{enumerate}

In our problem scenario, we additionally have a symmetry group $G$ acting on $\mc Q$.
We intend that the symmetry group describes configurations which are ``indistinguishable''.
For the pure path planning problem that we consider in this paper, this just means the distinct configurations must be geometrically indistinguishable: $\forall g\hspace{-0.1em}\in\hspace{-0.1em} G,q\hspace{-0.1em}\in\hspace{-0.1em}\mc Q$, $q\hspace{-0.1em}\in\hspace{-0.1em}\mc Q_\free\Leftrightarrow g\hspace{-0.1em}\cdot\hspace{-0.1em} q\hspace{-0.1em}\in\hspace{-0.1em}\mc Q_\free$.
If dynamics were involved, then the various dynamics properties (e.g. friction coefficients and moments of inertia) would also have to match across the symmetry.
We can adapt \eqref{eq:regular_motion_planning_problem}, and our new motion planning problem is
\begin{equation}
    \hspace{-10pt}\begin{array}{rl}
        \operatorname{find} & \gamma:[0,1]\to\mc Q/G\\
        \ptxt{such that} & \gamma(0)=[q_0],\;\gamma(1)=[q_1],\\
        & q_t\in\mc Q_\free,\;\forall q_t\in\pi\inv(\gamma(t)),\;\forall t\in[0,1].
    \end{array}
    \label{eq:quotient_motion_planning_problem}
\end{equation}
Note that the choice of $q_t\in\pi\inv(\gamma(t))$ is irrelevant, since we have assumed configurations are indistinguishable under the symmetry.
In a slight abuse of notation, we let $\mc Q_\free/G$ denote the subset of $\mc Q/G$ which is collision free.

Conceptually speaking, motion planning algorithms that solve \eqref{eq:regular_motion_planning_problem} can solve \eqref{eq:quotient_motion_planning_problem} by building an undirected graph in $\mc Q/G$, using closed-form solutions of the IMACS primitives.
\begin{enumerate}
    \item Distance metric: for $[q_0],[q_1]\in\mc Q/G$, fix $q_0\in[q_0]$ and find the $q_1\in[q_1]$ minimizing $\rmd_\mc Q(q_0,q_1)$.
    That value is $\rmd_{\mc Q/G}([q_0],[q_1])$.
    ($\rmd_\mc Q$ is known in closed-form for $\SO(2)$ and $\SO(3)$.)

    \item Global sampler: because $\pi:\mc Q\to\mc Q/G$ is measure-preserving, we can obtain uniform global samples in $\mc Q/G$ by uniformly sampling $\mc Q$ and projecting.
    
    \item Local sampler: we can draw samples from the ball of radius $r$ about $q$, $B_r([q])$, by choosing any $q'\in[q]$, sampling $B_r(q')\subseteq\mc Q$, and projecting.
    
    \item Local planner: when computing the distance between $[q_0]$ and $[q_1]$, we get a minimizing geodesic $\gamma$ connecting $q_0$ and $q_1$ in $\mc Q$.
    (E.g., using linear interpolation for Euclidean components, or spherical linear interpolation for $\SO(3)$~\cite{shoemake1985animating}.)
    If $\gamma$ is collision-free, then the local planner outputs $\pi\of\gamma$, which connects $[q_0]$ and $[q_1]$.
\end{enumerate}

Practically speaking, the algorithm will produce a \emph{directed} graph $\mc G=(\mc V,\mc E)$.
Each vertex has a corresponding $q_v\in\mc Q$, the \emph{canonical representative} of the equivalence class $[q_v]$.
Note that for each edge in the graph $e=(u,v)$, the minimizing geodesic connecting $q_u$ to $q_v$ in $\mc Q$ may not project to the shortest path between $[q_u]$ and $[q_v]$ in $\mc Q/G$.
So we associate to $e$ the \emph{end point} $q_e\in[q_v]$, chosen such that the minimizing geodesic connecting $q_u$ to $q_e$ projects to the shortest geodesic between $[q_u]$ and $[q_v]$.
This demonstrates why we must internally use a directed graph, as the reverse edge has a different end point.

Once the sampling-based planning algorithm returns a path in the graph, we lift it to a continuous trajectory in $\mc Q$.
A vertex path in the graph $(v_i)_{i=0}^N$ admits an edge sequence $(e_i)_{i=1}^N$ with $e_i=(v_{i-1},v_i)$.
We cannot directly concatenate the edges, as for a given $e_i$, we may have $q_{e_i}\ne q_{v_i}$.
So for each vertex, we find the $g\in G$ such that $q_{e_i}=g\cdot q_{v_i}$, and then append $g\cdot q_{e_{i+1}}$.
Thus, we obtain a sequence of points $q_0',\ldots,q_N'\in\mc Q$ such that the piecewise minimizing geodesic connecting these points has the same length as the path in $\mc Q/G$ described by the edge path in the graph.

\subsection{Analysis of Sampling-Based Planning Algorithms}
\label{sec:methodology:analysis_of_sampling_based_planning_algorithms}

In this section, we extend several standard bounds on the performance of common sampling-based motion algorithms, simultaneously verifying probabilistic completeness and quantifying the improvement that results from leveraging symmetries.
We emphasize that in each case, the original bound is recovered when $\card{G}=1$, and note that our results here are only improving on bounds, so they do not guarantee that the algorithms will take fewer samples.
We prove:
\begin{enumerate}
    \item a constant-factor improvement of $\card{G}$ on the exponent in an existing bound on the probability that an RRT has not found a solution after $k$ iterations,
    \item a multiplicative improvement of $1/\card{G}$ on the the expected number of samples needed for a PRM to find a path, and
    \item the connection radii needed for RRT* and PRM* to be asymptotically optimal drops by a factor of $\card{G}^{1/(d+1)}$ and $\card{G}^{1/d}$ (respectively), where $d$ is the dimension of the configuration space.
\end{enumerate}

We begin by stating some readily-apparent geometric properties about the configuration space.
(Proofs of these results are deferred to the Appendix.)
\begin{lemma}
    $\Vol(\mc Q/G)=\Vol(\mc Q)/\card{G}$.
    \phantomsection
    \label{lem:volume}
\end{lemma}
\begin{corollary}
    $\Vol(\mc Q_\free/G)=\Vol(\mc Q_\free)/\card{G}$.
    \label{cor:free_volume}
\end{corollary}
\begin{lemma}
    \textbf{(Cheeger and Ebin)}
    Fix $\eps<r_\mrm{inj}(\mc Q/G)$ and $q\in\mc Q$. Then $\restr{\pi}_{B_\eps(q)}$ is a diffeomorphism onto its image~\cite{cheeger1975comparison}.\vspace{-\baselineskip}
    \label{lem:balls_to_balls}
\end{lemma}
\Cref{lem:volume} reveals how much the volume of the configuration space shrinks when leveraging symmetries, and \Cref{lem:balls_to_balls} gives a notion of how large a scale we can look at before the global topology plays a role.
Together, they yield:
\begin{corollary}
    Fix $q\in\mc Q$, $\eps<r_\mrm{inj}(\mc Q/G)$.
    If $q'$ is uniformly sampled from $\mc Q$, then the probability $\Prob[q'\in B_\eps(q)]=\Prob[[q']\in B_\eps([q])]/\card{G}$.
    \label{cor:ball_probability}
\end{corollary}
Computing the injectivity radius of a manifold is nontrivial, but we can find a lower bound on $r_\mrm{inj}(\mc Q/G)$.
\begin{theorem}
    $r_\mrm{inj}(\mc Q/G)\ge r_\mrm{inj}(\mc Q)/\card{G}$.
    \label{thm:injectivity_radius}
\end{theorem}
This bound is tight for cyclic groups in $\SO(2)$ and $\SO(3)$.
Because the injectivity radius is known for $\SO(2)$, $\SO(3)$, and the unit $d$-sphere, we can bound $r_\mrm{inj}(\mc Q/G)$ for all configuration spaces of interest.

Now, we turn our attention to the notion of \emph{clearance}.
A path $\gamma:[0,1]\to\mc Q$ is \emph{$\delta$-clear} if $\forall t\in[0,1]$, $B_\delta(\gamma(t))\subseteq\mc Q_\free$.
We leverage the following corollary of \Cref{lem:balls_to_balls}.
\begin{corollary}
    If $\gamma:[0,1]\to\mc Q$ is $\delta$-clear with $\delta<r_\mrm{inj}(\mc Q/G)$, then $\pi\of\gamma$ is $\delta$-clear in $\mc Q/G$.
    \label{cor:clearance}
\end{corollary}
The upper bound on $\delta$ is to ensure that the radius of the ball is preserved by $\pi$.
This demonstrates that our bounds are only applicable when the clearance of the trajectory is on a smaller scale than the topological changes induced by the symmetries.
If this is not the case, the obstacles must be sparse, suggesting an easy underlying planning problem.

For the following results, suppose we are planning from $q_0$ to $q_1$, which are connected by a $\delta$-clear path $\gamma:[0,1]\to\mc Q$.
Suppose further that $\delta<r_\mrm{inj}(\mc Q/G)$, so $\pi\of\gamma$ is a $\delta$-clear path connecting $[q_0]$ to $[q_1]$.
Finally, let $\ell=\ms L(\gamma)=\ms L(\pi\of\gamma)$.

\subsubsection{RRT Sample Complexity} Kleinbort et al.~\cite{kleinbort2018probabilistic} presented a new proof of the probabilistic completeness of the RRT algorithm by bounding the failure probability for a given number of samples.
Let $m=\frac{5\ell}{\nu}$, $\nu=\min\set{\delta,\eta}$, where $\eta$ is the step size used by the RRT planner, and $p$ is the probability that a uniform sample from $\mc Q$ falls into a ball of radius $\nu/5$.
In the symmetry aware case, \Cref{cor:ball_probability} yields a new probability $p'=p\card{G}$, but the $\delta$-clearance and path length are unchanged, so we obtain the bound
\begin{equation}
    \Prob[\ptxt{Not Reached in $k$ Iterations}]\hspace{-0.2em}\le\hspace{-0.2em}\frac{1}{(m\hspace{-0.1em}-\hspace{-0.1em}1)!}k^mme^{\hspace{-0.1em}-\hspace{-0.1em}\card{G}pk}\hspace{-0.1em}.\hspace{-0.1em}
    \label{eq:new_rrt_sample_bound}
\end{equation}

\subsubsection{PRM Sampling Complexity}
Ladd and Kavraki~\cite{ladd2004measure} derived a new bound on the expected number of samples for a PRM to find a path.
Let $N$ be the number of samples that are drawn when a PRM repeatedly samples until it finds a path from $q_0$ to $q_1$.
$H(n)$ is the $n$th Harmonic number and $\Vol(B_{\delta/2}(\cdot))$ is the volume of the ball of radius $\delta/2$.
In the symmetry-aware case, $\ell$, $\delta$, and $\Vol(B_{\delta/2}(\cdot))$ are unchanged.
However, $\Vol(\mc Q_\free)$ decreases by a factor of $\card{G}$ due to \Cref{cor:free_volume}, yielding the bound
\begin{equation}
    \E[N]\le\frac{H\paren{\frac{2\ell}{\delta}}\Vol(\mc Q_\free)}{\card{G}\Vol(B_{\delta/2}(\cdot))}.
    \label{eq:new_prm_sample_bound}
\end{equation}

\begin{figure}
	\centering
	\includegraphics[height=6cm]{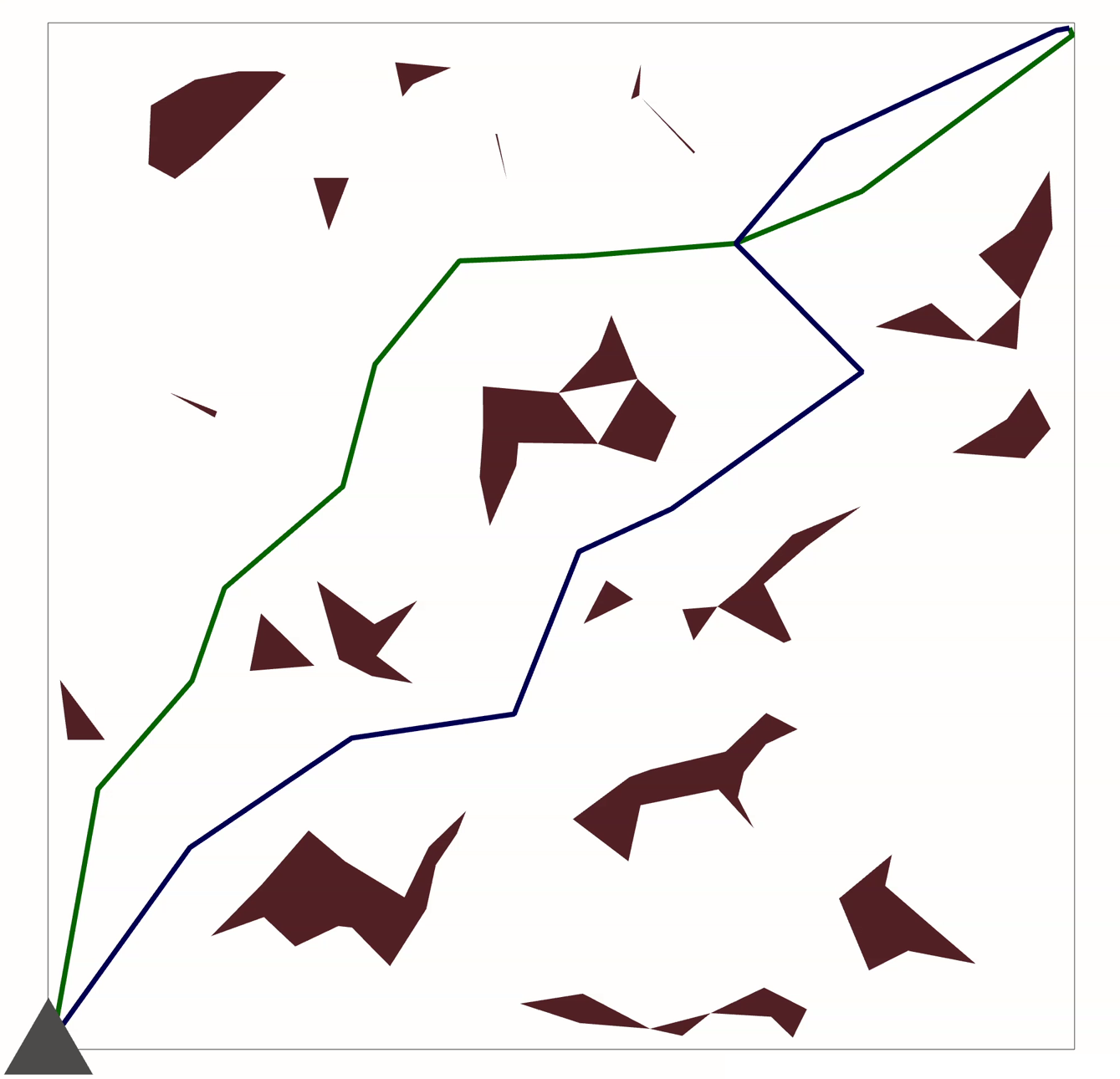}
	\caption{
		A randomly-generated world in 2D. The green and blue paths are symmetry-aware and -unaware paths for the triangle (respectively), found using KNN-PRM*.
		\vspace{-\baselineskip}
	}
	\label{fig:experimental_setup_2d}
\end{figure}

\subsubsection{PRM* Connection Radius}
When Karaman and Frazzoli~\cite{karaman2011sampling} introduced PRM*, they also bounded the minimum connection radius scaling parameter $\rho$ necessary to achieve asymptotic optimality.
The only geometric construction necessary is the sequence of covering balls~\cite[C.3]{karaman2011sampling}, and the radius of each ball is bounded above by the clearance of the path (which, in turn, is smaller than the injectivity radius).
Thus, for the symmetry-aware Radius-PRM* planner to be asymptotically optimal, the connection radius must scale by
\begin{equation}
    \rho_\mrm{prm}>2(1+1/d)^{1/d}\paren{\frac{\Vol(\mc Q_\free)}{\Vol(B_1(\cdot))}}^{1/d}\card{G}^{-1/d}.
    \label{eq:new_prm_star_radius}
\end{equation}

\subsubsection{RRT* Connection Radius}
For the connection radius of RRT*, we build upon the bound of Solovey et al.~\cite{solovey2020revisiting}, which was needed to close a logical gap in Karaman and Frazzoli's original proof of asymptotic optimality.
Once again, the only geometric construction necessary is a sequence of covering balls of small radius, so no values in the bound change except the volume of the whole space.
Let $c^*$ be the length of the shortest path connecting $q_0$ to $q_1$, $\theta\in(0,\frac{1}{4})$, and $\mu\in(0,1)$.
Then for the symmetry-aware RRT* planner, the probability of a solution with cost at most $(1+\epsilon)c^*$ for $\epsilon\in(0,1)$ approaches $1$ as the number of samples increases if
\begin{equation}
    \rho_\mrm{rrt}\ge(2+\theta)\paren{\frac{(1+\epsilon/4)c^*\Vol(\mc Q_\free)}{(d+1)\theta(1-\mu)\Vol(B_1(\cdot))}}^{\frac{1}{d+1}}\card{G}^{\frac{-1}{d+1}}.
    \label{eq:new_rrt_star_radius}
\end{equation}

\section{Experiments}
\label{sec:experiments}
\begin{figure}
	\centering
	\includegraphics[height=6cm]{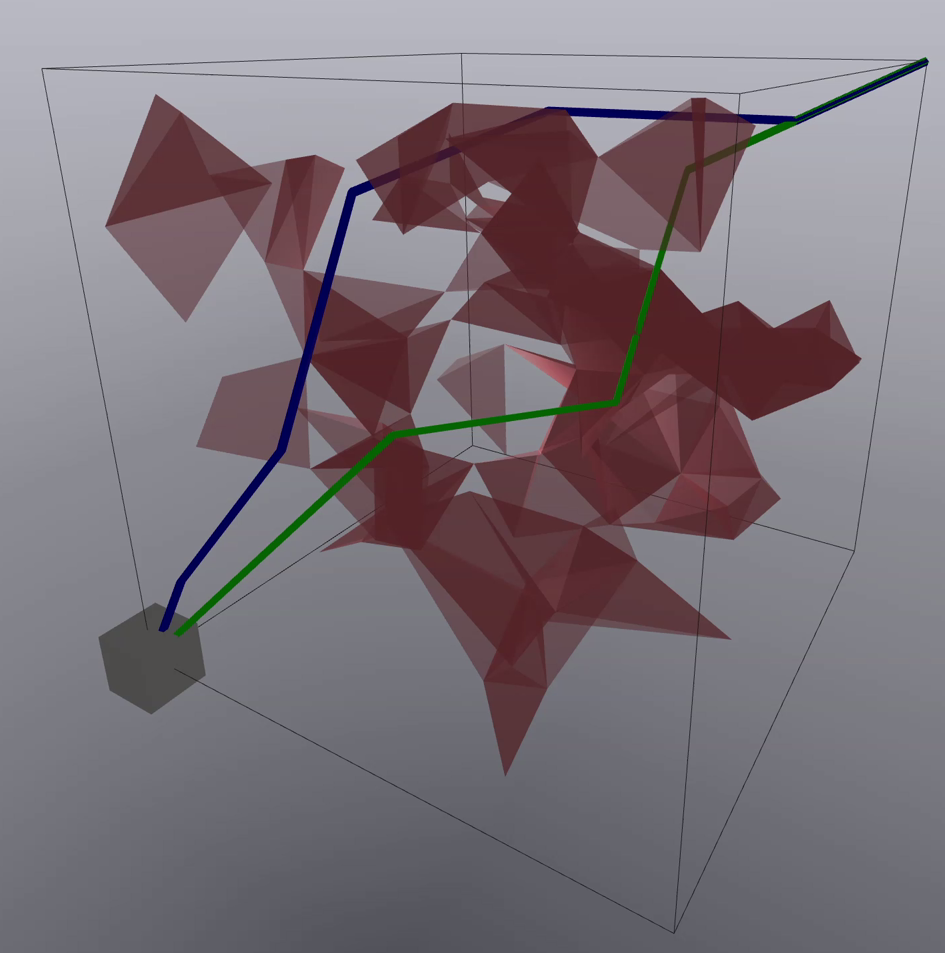}
	\caption{
		A randomly-generated world in 3D. The green and blue paths are symmetry-aware and -unaware paths for the cube (respectively), found using RRT*.
		\vspace{-\baselineskip}
	}
	\label{fig:experimental_setup_3d}
\end{figure}

\begin{table*}
	\centering
	\begin{tabular}{|>{\centering\arraybackslash}p{1.25cm}|>{\centering\arraybackslash}p{1.25cm}|>{\centering\arraybackslash}p{1.1cm}|>{\centering\arraybackslash}p{1.55cm}|>{\centering\arraybackslash}p{1.55cm}|>{\centering\arraybackslash}p{1.9cm}|>{\centering\arraybackslash}p{1.55cm}|>{\centering\arraybackslash}p{1.9cm}|>{\centering\arraybackslash}p{1.7cm}|} \hline
		Workspace Dimension & Object & Symmetry Group Order & RRT Online Runtime Improvement & RRT Path Length Improvement & RRT* Online Runtime Improvement & RRT* Path Length Improvement & KNN-PRM* Offline Runtime Improvement & KNN-PRM* Path Length Improvement\\ \hline
		2 & Triangle & 3 & 2.766x & 1.216x & 0.828x & 1.133x & 0.832x & 1.069x\\ \hline
		2 & Pentagon & 5 & 2.432x & 1.187x & 0.718x & 1.141x & 0.614x & 1.107x\\ \hline
		2 & Octagon & 8 & 2.192x & 1.227x & 0.611x & 1.157x & 0.367x & 1.116x\\ \hline
		3 & 8-Pyramid & 8 & 2.016x & 1.179x & 0.844x & 1.091x & 0.347x & 1.157x\\ \hline
		3 & 6-Prism & 12 & 2.951x & 1.279x & 0.779x & 1.146x & 0.233x & 1.190x\\ \hline
		3 & Tetrahedron & 12 & 2.360x & 1.310x & 0.751x & 1.163x & 0.227x & 1.193x\\ \hline
		3 & Cube & 24 & 2.759x & 1.364x & 0.749x & 1.198x & 0.100x & 1.196x\\ \hline
	\end{tabular}
	\caption{
		Performance of symmetry-aware algorithms relative to symmetry-unaware baselines given equal resources. %
		For each planner, we show the relative speedup and the relative path length improvement. Higher numbers represent greater speedups and shorter paths, with 1.0 indicating equal performance.
		Due to space constraints, additional experiments and complete results (including Radius-PRM*) are available online at \href{https://cohnt.github.io/projects/symmetries.html}{https://cohnt.github.io/projects/symmetries.html}.
	}
	\label{tab:same_resources}
\end{table*}

\begin{table*}
	\centering
	\begin{tabular}{|>{\centering\arraybackslash}p{1.25cm}|>{\centering\arraybackslash}p{1.25cm}|>{\centering\arraybackslash}p{1.1cm}|>{\centering\arraybackslash}p{1.55cm}|>{\centering\arraybackslash}p{1.55cm}|>{\centering\arraybackslash}p{1.9cm}|>{\centering\arraybackslash}p{1.55cm}|>{\centering\arraybackslash}p{1.9cm}|>{\centering\arraybackslash}p{1.7cm}|} \hline
		Workspace Dimension & Object & Symmetry Group Order & RRT* Online Runtime Improvement & RRT* Path Length Improvement & KNN-PRM* Offline Runtime Improvement & KNN-PRM* Path Length Improvement & Radius-PRM* Offline Runtime Improvement & Radius-PRM* Path Length Improvement\\ \hline
		2 & Triangle & 3 & 0.845x & 1.133x & 3.005x & 0.992x & 2.519x & 1.003x\\ \hline
		2 & Pentagon & 5 & 0.744x & 1.141x & 5.103x & 1.002x & 7.414x & 1.090x\\ \hline
		2 & Octagon & 8 & 0.640x & 1.157x & 6.570x & 1.031x & 15.225x & 1.086x\\ \hline
		3 & 8-Pyramid & 8 & 0.845x & 1.091x & 7.057x & 1.058x & 15.776x & 0.903x\\ \hline
		3 & 6-Prism & 12 & 0.777x & 1.146x & 9.849x & 1.092x & 20.577x & 0.989x\\ \hline
		3 & Tetrahedron & 12 & 0.751x & 1.163x & 10.276x & 1.138x & 20.501x & 1.032x\\ \hline
		3 & Cube & 24 & 0.749x & 1.198x & 16.216x & 1.125x & 32.496x & 1.029x\\ \hline
	\end{tabular}
	\caption{
		Relative performance of symmetry-aware algorithms when given reduced resources than the symmetry-unaware baselines.
		Metrics are the same as in \Cref{tab:same_resources}.
		(RRT is omitted as no strategy for budgeting its available resources is presented.)
		\vspace{-2\baselineskip}
	}
	\label{tab:uneven_resources}
\end{table*}

To numerically evaluate our approach and verify the practical applicability of the sample complexity improvements, we present a comprehensive suite of experiments.
We consider four sampling-based planning algorithms: RRT, RRT*, KNN-PRM*, and Radius-PRM*.
We compare each algorithm with a symmetry-unaware baseline, that simply plans between arbitrary representatives of the start and goal configurations.
Unless stated otherwise, we always set connection radius and neighborhood size to the minimum values necessary to guarantee asymptotic optimality.
Because the performance of these algorithms varies greatly, we focus on running a large number of experiments across different environments and different start/goal pairs.

We generate random worlds by uniformly sampling points from either 2D or 3D within prescribed limits, and computing an alpha shape~\cite{edelsbrunner1983shape}.
Example setups are shown in~\Cref{fig:experimental_setup_2d,fig:experimental_setup_3d}.
Drake~\cite{tedrake2019drake} is used for collision checking.
When comparing a symmetry-aware planner with its corresponding baseline, we use 10 random worlds, with 100 random start/goal pairs, and compare the results in terms of success rate, runtime, and path length.
Additional details (more objects, standard deviations, and success rates) are available online at \href{https://cohnt.github.io/projects/symmetries.html}{https://cohnt.github.io/projects/symmetries.html}.

In \Cref{tab:same_resources}, we consider what happens when the symmetry-aware and -unaware planners are given the same computational resources.
The RRT and RRT* planners are capped at 1000 samples in 2D and 250 samples in 3D, and we use the length of the path found by the symmetry-unaware RRT as an overapproximation of $c^*$ in \eqref{eq:new_rrt_star_radius} when determining the RRT* radius.
The PRM* planners use $3000\card{G}$ samples in 2D and $500\card{G}$ samples in 3D.
We omitted planning failures from these averages, as they are generally caused by the configuration space being disconnected, and the failure rates of the symmetry-aware and -unaware planners were similar.

The symmetry-aware RRT planner has a strong performance across the board, consistently producing paths that are 20-40\% shorter than those from the symmetry-unaware planner, in less than half the time.
RRT* and KNN-PRM* also find shorter paths, but generally require more time, due to the increased computational complexity of computing distances.
Radius-PRM* scales particularly poorly, as the shrinkage of the configuration space leads to more possible edges within the given radius.

In \Cref{tab:uneven_resources}, we give the symmetry-aware planners reduced computational resources, according to the results in \Cref{sec:methodology:analysis_of_sampling_based_planning_algorithms}.
For Radius-PRM* and RRT*, we shrink the connection radius by a factor of $1/\card{G}^{1/d}$ and $1/\card{G}^{1/(d+1)}$ due to \eqref{eq:new_prm_star_radius} and \eqref{eq:new_rrt_star_radius}.
The connection radius for RRT* is further reduced since the shorter path found by the symmetry-aware planner is a tighter upper bound for $c^*$.
For KNN-PRM* and Radius-PRM*, we reduce the number of samples by a factor of $\card{G}$ due to \eqref{eq:new_prm_sample_bound}.
No reduction on the neighborhood size $k$ is possible for KNN-PRM*, as $k$ is only dimension-dependent~\cite[\S 4.2]{karaman2011sampling}.
RRT* does not get runtime improvements, likely because the smaller connection radius does not outweigh the decrease in volume of the configuration space (which leads to samples being closer together).

Finally, to study the scaling with respect to dimension, we consider planning problems with multiple objects with symmetries, for which we must plan paths jointly.
We use a bidirectional RRT (BiRRT)~\cite{kuffner2000rrt} (capped at 40000 samples).
In \Cref{tab:dimension_scaling}, we show the relative improvement in the runtime of the symmetry-aware methodologies, as the dimension of C-space increases.
Although the number of samples required still grows exponentially with the dimension, the constant factor reduction in volume of every component still helps the symmetry-aware planners to scale to higher dimensions than the symmetry-unaware baselines.
(As discussed in \Cref{sec:background:symmetry_groups}, the complexity of computing distances grows linearly with the number of copies, even as the order of the symmetry group grows exponentially.)

\section{Discussion}
\label{sec:discussion}
In this paper, we have presented a general strategy for adapting sampling-based planning algorithms to handle discrete configuration-space symmetries.
We demonstrate that the geometry of such spaces is amenable to efficient planning through the lens of the IMACS framework, as all necessary primitives can be written in closed-form.
Rigorous theoretical results demonstrate improvements in sample-complexity, and comprehensive experimental results verify the theory leads to dramatically improved runtimes for the popular RRT, KNN-PRM*, and Radius-PRM* planning algorithms.

One direction for future work is studying the performance of other sampling-based planners with symmetries.
Our analysis is mostly based on properties of the quotient space, so these other approaches should see similar improvements, even in the kinodynamic setting.

Another direction for future work is extending our results to trajectory optimization, as there are a variety of robotics tasks (especially in contact-rich manipulation) where such methods have proved more performant than sampling-based planning.
Hybrid methods that leverage sampling and trajectory optimization together~\cite{pang2023global,suh2025ctr} could naturally be generalized to handle symmetries purely at the roadmap level.
And contact-rich planning methods built on Graphs of Convex Sets~\cite{graesdal2024towards,chia2024gcs} are suited to modeling symmetries, given the ease with which the symmetry can be reduced to just another discrete decision.

\begin{table}
    \centering
    \setlength{\tabcolsep}{4pt} %
    \begin{tabular}{|>{\centering\arraybackslash}p{0.9cm}|>{\centering\arraybackslash}p{1.45cm}|>{\centering\arraybackslash}p{1.1cm}|>{\centering\arraybackslash}p{1.7cm}|>{\centering\arraybackslash}p{1.57cm}|} \hline
        Number of Objects & Configuration Space Dimension & Symmetry Group Order & BiRRT Online Runtime Improvement & BiRRT Path Length Improvement\\ \hline
        1 & 3 & 2 & 1.121x & 1.064x\\ \hline
        2 & 6 & 4 & 1.350x & 1.099x\\ \hline
        3 & 9 & 8 & 1.824x & 1.131x\\ \hline
        4 & 12 & 16 & 2.573x & 1.137x\\ \hline
        5 & 15 & 32 & 2.747x & 1.101x\\ \hline
    \end{tabular}
    \caption{
        Performance of symmetry-aware BiRRT as the configuration space grows in dimension, with multiple copies of a rectangle. The configuration space is $\SE(2)^m$, the the symmetry group is $(C_2)^m$, where $m$ is the number of copies.
    }
    \label{tab:dimension_scaling}
\end{table}

\section{Acknowledgements}
\label{sec:acknowledgements}
The authors thank Seiji Shaw and Adam Wei (Massachusetts Institute of Technology), HJ Terry Suh (CarbonSix), Tao Pang (Robotics and AI Institute), Zachary Kingston (Purdue University), and Laura Weidensager (Chemnitz University of Technology).
The authors acknowledge the MIT SuperCloud and Lincoln Laboratory Supercomputing Center for providing HPC resources that have contributed to the research results reported within this paper.
This work was supported by Amazon.com, PO No. 2D-06310236 and the National Science Foundation Graduate Research Fellowship Program under Grant No. 2141064.
Any opinions, findings, and conclusions or recommendations expressed in this material are those of the author(s) and do not necessarily reflect the views of the National Science Foundation.

\bibliographystyle{IEEEtran}
\bibliography{ref.bib}

\begin{thebibliography}{10}
\providecommand{\url}[1]{#1}
\csname url@rmstyle\endcsname
\providecommand{\newblock}{\relax}
\providecommand{\bibinfo}[2]{#2}
\providecommand\BIBentrySTDinterwordspacing{\spaceskip=0pt\relax}
\providecommand\BIBentryALTinterwordstretchfactor{4}
\providecommand\BIBentryALTinterwordspacing{\spaceskip=\fontdimen2\font plus
\BIBentryALTinterwordstretchfactor\fontdimen3\font minus
  \fontdimen4\font\relax}
\providecommand\BIBforeignlanguage[2]{{%
\expandafter\ifx\csname l@#1\endcsname\relax
\typeout{** WARNING: IEEEtran.bst: No hyphenation pattern has been}%
\typeout{** loaded for the language `#1'. Using the pattern for}%
\typeout{** the default language instead.}%
\else
\language=\csname l@#1\endcsname
\fi
#2}}

\bibitem{dafle2014extrinsic}
N.~C. Dafle, A.~Rodriguez, R.~Paolini, B.~Tang, S.~S. Srinivasa, M.~Erdmann,
  M.~T. Mason, I.~Lundberg, H.~Staab, and T.~Fuhlbrigge, ``Extrinsic dexterity:
  In-hand manipulation with external forces,'' in \emph{2014 IEEE International
  Conference on Robotics and Automation (ICRA)}.\hskip 1em plus 0.5em minus
  0.4em\relax IEEE, 2014, pp. 1578--1585.

\bibitem{chavan2018hand}
N.~Chavan-Dafle, R.~Holladay, and A.~Rodriguez, ``In-hand manipulation via
  motion cones,'' \emph{arXiv preprint arXiv:1810.00219}, 2018.

\bibitem{pang2023global}
T.~Pang, H.~T. Suh, L.~Yang, and R.~Tedrake, ``Global planning for contact-rich
  manipulation via local smoothing of quasi-dynamic contact models,''
  \emph{IEEE Transactions on robotics}, 2023.

\bibitem{suh2025ctr}
H.~Suh, T.~Pang, T.~Zhao, and R.~Tedrake, ``Dexterous contact-rich manipulation
  via the contact trust region,'' \emph{arXiv preprint arXiv:2505.02291}, 2025.

\bibitem{ghandi2015review}
S.~Ghandi and E.~Masehian, ``Review and taxonomies of assembly and disassembly
  path planning problems and approaches,'' \emph{Computer-Aided Design},
  vol.~67, pp. 58--86, 2015.

\bibitem{lozano1990spatial}
T.~Lozano-Perez, \emph{Spatial planning: A configuration space approach}.\hskip
  1em plus 0.5em minus 0.4em\relax Springer, 1990.

\bibitem{berenson2009manipulation}
D.~Berenson, S.~S. Srinivasa, D.~Ferguson, and J.~J. Kuffner, ``Manipulation
  planning on constraint manifolds,'' in \emph{2009 IEEE international
  conference on robotics and automation}.\hskip 1em plus 0.5em minus
  0.4em\relax IEEE, 2009, pp. 625--632.

\bibitem{biggs2018motion}
J.~D. Biggs and H.~C. Henninger, ``Motion planning on a class of {6-D Lie}
  groups via a covering map,'' \emph{IEEE Transactions on Automatic Control},
  vol.~64, no.~9, pp. 3544--3554, 2018.

\bibitem{kingston2018sampling}
Z.~Kingston, M.~Moll, and L.~E. Kavraki, ``Sampling-based methods for motion
  planning with constraints,'' \emph{Annual review of control, robotics, and
  autonomous systems}, vol.~1, pp. 159--185, 2018.

\bibitem{geist2024learning}
A.~R. Geist, J.~Frey, M.~Zobro, A.~Levina, and G.~Martius, ``Learning with 3d
  rotations, a hitchhiker's guide to {SO(3)},'' \emph{arXiv preprint
  arXiv:2404.11735}, 2024.

\bibitem{lee2012smooth}
J.~M. Lee, \emph{Smooth manifolds}.\hskip 1em plus 0.5em minus 0.4em\relax
  Springer, 2012.

\bibitem{kingston2019exploring}
Z.~Kingston, M.~Moll, and L.~E. Kavraki, ``Exploring implicit spaces for
  constrained sampling-based planning,'' \emph{The International Journal of
  Robotics Research}, vol.~38, no. 10-11, pp. 1151--1178, 2019.

\bibitem{zucker2013chomp}
M.~Zucker, N.~Ratliff, A.~D. Dragan, M.~Pivtoraiko, M.~Klingensmith, C.~M.
  Dellin, J.~A. Bagnell, and S.~S. Srinivasa, ``{CHOMP}: Covariant hamiltonian
  optimization for motion planning,'' \emph{The International Journal of
  Robotics Research}, vol.~32, no. 9-10, pp. 1164--1193, 2013.

\bibitem{kalakrishnan2011stomp}
M.~Kalakrishnan, S.~Chitta, E.~Theodorou, P.~Pastor, and S.~Schaal, ``{STOMP}:
  Stochastic trajectory optimization for motion planning,'' in \emph{2011 IEEE
  international conference on robotics and automation}.\hskip 1em plus 0.5em
  minus 0.4em\relax IEEE, 2011, pp. 4569--4574.

\bibitem{toussaint2014komo}
M.~Toussaint, ``Newton methods for k-order {Markov} constrained motion
  problems,'' \emph{arXiv preprint arXiv:1407.0414}, 2014.

\bibitem{howell2019altro}
T.~A. Howell, B.~E. Jackson, and Z.~Manchester, ``{ALTRO}: A fast solver for
  constrained trajectory optimization,'' in \emph{2019 IEEE/RSJ International
  Conference on Intelligent Robots and Systems (IROS)}.\hskip 1em plus 0.5em
  minus 0.4em\relax IEEE, 2019, pp. 7674--7679.

\bibitem{manipulation}
\BIBentryALTinterwordspacing
R.~Tedrake, \emph{Robotic Manipulation}, 2024. [Online]. Available:
  \url{http://manipulation.mit.edu}
\BIBentrySTDinterwordspacing

\bibitem{sundaralingam2023curobo}
B.~Sundaralingam, S.~K.~S. Hari, A.~Fishman, C.~Garrett, K.~Van~Wyk, V.~Blukis,
  A.~Millane, H.~Oleynikova, A.~Handa, F.~Ramos, \emph{et~al.}, ``Curobo:
  Parallelized collision-free robot motion generation,'' in \emph{2023 IEEE
  International Conference on Robotics and Automation (ICRA)}.\hskip 1em plus
  0.5em minus 0.4em\relax IEEE, 2023, pp. 8112--8119.

\bibitem{marcucci2023motion}
T.~Marcucci, M.~Petersen, D.~von Wrangel, and R.~Tedrake, ``Motion planning
  around obstacles with convex optimization,'' \emph{Science robotics}, vol.~8,
  no.~84, p. eadf7843, 2023.

\bibitem{lavalle1998rapidly}
S.~M. LaValle, ``Rapidly-exploring random trees: A new tool for path
  planning,'' 1998.

\bibitem{kavraki1996probabilistic}
L.~E. Kavraki, P.~Svestka, J.-C. Latombe, and M.~H. Overmars, ``Probabilistic
  roadmaps for path planning in high-dimensional configuration spaces,''
  \emph{IEEE transactions on Robotics and Automation}, vol.~12, no.~4, pp.
  566--580, 1996.

\bibitem{orthey2023sampling}
A.~Orthey, C.~Chamzas, and L.~E. Kavraki, ``Sampling-based motion planning: A
  comparative review,'' \emph{Annual Review of Control, Robotics, and
  Autonomous Systems}, vol.~7, 2023.

\bibitem{teng2022error}
S.~Teng, D.~Chen, W.~Clark, and M.~Ghaffari, ``An error-state model predictive
  control on connected matrix {Lie} groups for legged robot control,'' in
  \emph{2022 IEEE/RSJ International Conference on Intelligent Robots and
  Systems (IROS)}.\hskip 1em plus 0.5em minus 0.4em\relax IEEE, 2022, pp.
  8850--8857.

\bibitem{fox2002extending}
M.~Fox and D.~Long, ``Extending the exploitation of symmetries in planning,''
  in \emph{Proceedings of Sixth International Conference on AI Planning and
  Scheduling}, 2002, pp. 83--91.

\bibitem{domshlak2013symmetry}
C.~Domshlak, M.~Katz, and A.~Shleyfman, ``Symmetry breaking: Satisficing
  planning and landmark heuristics,'' in \emph{Proceedings of the International
  Conference on Automated Planning and Scheduling}, vol.~23, 2013, pp.
  298--302.

\bibitem{pedrosa2021graph}
M.~V. Pedrosa, T.~Schneider, and K.~Fla{\ss}kamp, ``Graph-based motion planning
  with primitives in a continuous state space search,'' in \emph{2021 6th
  International Conference on Mechanical Engineering and Robotics Research
  (ICMERR)}.\hskip 1em plus 0.5em minus 0.4em\relax IEEE, 2021, pp. 30--39.

\bibitem{frazzoli2005maneuver}
E.~Frazzoli, M.~A. Dahleh, and E.~Feron, ``Maneuver-based motion planning for
  nonlinear systems with symmetries,'' \emph{IEEE transactions on robotics},
  vol.~21, no.~6, pp. 1077--1091, 2005.

\bibitem{cheng2003exploiting}
P.~Cheng, E.~Frazzoli, and S.~M. LaValle, ``Exploiting group symmetries to
  improve precision in kinodynamic and nonholonomic planning,'' in
  \emph{Proceedings 2003 IEEE/RSJ International Conference on Intelligent
  Robots and Systems (IROS 2003)(Cat. No. 03CH37453)}, vol.~1.\hskip 1em plus
  0.5em minus 0.4em\relax IEEE, 2003, pp. 631--636.

\bibitem{orthey2024multilevel}
A.~Orthey, S.~Akbar, and M.~Toussaint, ``Multilevel motion planning: A fiber
  bundle formulation,'' \emph{The international journal of robotics research},
  vol.~43, no.~1, pp. 3--33, 2024.

\bibitem{garrett2021integrated}
C.~R. Garrett, R.~Chitnis, R.~Holladay, B.~Kim, T.~Silver, L.~P. Kaelbling, and
  T.~Lozano-P{\'e}rez, ``Integrated task and motion planning,'' \emph{Annual
  review of control, robotics, and autonomous systems}, vol.~4, no.~1, pp.
  265--293, 2021.

\bibitem{scott2018trajectory}
K.~M. Scott, W.~C. Barott, and B.~Himed, ``Trajectory optimization with
  discrete decisions,'' in \emph{2018 IEEE/AIAA 37th Digital Avionics Systems
  Conference (DASC)}.\hskip 1em plus 0.5em minus 0.4em\relax IEEE, 2018, pp.
  1--7.

\bibitem{bordalba2022direct}
R.~Bordalba, T.~Schoels, L.~Ros, J.~M. Porta, and M.~Diehl, ``Direct
  collocation methods for trajectory optimization in constrained robotic
  systems,'' \emph{IEEE Transactions on Robotics}, vol.~39, no.~1, pp.
  183--202, 2022.

\bibitem{nash1956imbedding}
J.~Nash, ``The imbedding problem for riemannian manifolds,'' \emph{Annals of
  mathematics}, vol.~63, no.~1, pp. 20--63, 1956.

\bibitem{arnold2018statistics}
R.~Arnold, P.~E. Jupp, and H.~Schaeben, ``Statistics of ambiguous rotations,''
  \emph{Journal of Multivariate Analysis}, vol. 165, pp. 73--85, 2018.

\bibitem{hielscher2021locally}
R.~Hielscher and L.~Lippert, ``Locally isometric embeddings of quotients of the
  rotation group modulo finite symmetries,'' \emph{Journal of Multivariate
  Analysis}, vol. 185, p. 104764, 2021.

\bibitem{zhu2023robot}
X.~Zhu, D.~Wang, G.~Su, O.~Biza, R.~Walters, and R.~Platt, ``On robot grasp
  learning using equivariant models,'' \emph{Autonomous Robots}, vol.~47,
  no.~8, pp. 1175--1193, 2023.

\bibitem{van2019geometric}
P.~van Goor, R.~Mahony, T.~Hamel, and J.~Trumpf, ``A geometric observer design
  for visual localisation and mapping,'' in \emph{2019 IEEE 58th Conference on
  Decision and Control (CDC)}.\hskip 1em plus 0.5em minus 0.4em\relax IEEE,
  2019, pp. 2543--2549.

\bibitem{zhu2022correspondence}
M.~Zhu, M.~Ghaffari, and H.~Peng, ``Correspondence-free point cloud
  registration with {SO(3)}-equivariant implicit shape representations,'' in
  \emph{Conference on robot learning}.\hskip 1em plus 0.5em minus 0.4em\relax
  PMLR, 2022, pp. 1412--1422.

\bibitem{zhao2024equivariant}
L.~Zhao, O.~Howell, X.~Zhu, J.~Y. Park, Z.~Zhang, R.~Walters, and L.~L. Wong,
  ``Equivariant action sampling for reinforcement learning and planning,''
  \emph{arXiv preprint arXiv:2412.12237}, 2024.

\bibitem{yang2024equibot}
J.~Yang, Z.-a. Cao, C.~Deng, R.~Antonova, S.~Song, and J.~Bohg, ``{EquiBot}:
  {SIM(3)}-equivariant diffusion policy for generalizable and data efficient
  learning,'' \emph{arXiv preprint arXiv:2407.01479}, 2024.

\bibitem{artin2010algebra}
M.~Artin, \emph{Algebra}, 2nd~ed.\hskip 1em plus 0.5em minus 0.4em\relax Upper
  Saddle River, NJ: Pearson, 2010.

\bibitem{lee2018introduction}
J.~M. Lee, \emph{Introduction to Riemannian manifolds}.\hskip 1em plus 0.5em
  minus 0.4em\relax Springer, 2018, vol.~2.

\bibitem{klein2003lectures}
F.~Klein, \emph{Lectures on the Icosahedron and the Solution of Equations of
  the Fifth Degree}.\hskip 1em plus 0.5em minus 0.4em\relax Courier
  Corporation, 2003.

\bibitem{shoemake1985animating}
K.~Shoemake, ``Animating rotation with quaternion curves,'' in
  \emph{Proceedings of the 12th annual conference on Computer graphics and
  interactive techniques}, 1985, pp. 245--254.

\bibitem{cheeger1975comparison}
J.~Cheeger, D.~G. Ebin, and D.~G. Ebin, \emph{Comparison theorems in Riemannian
  geometry}.\hskip 1em plus 0.5em minus 0.4em\relax North-Holland publishing
  company Amsterdam, 1975, vol.~9.

\bibitem{kleinbort2018probabilistic}
M.~Kleinbort, K.~Solovey, Z.~Littlefield, K.~E. Bekris, and D.~Halperin,
  ``Probabilistic completeness of {RRT} for geometric and kinodynamic planning
  with forward propagation,'' \emph{IEEE Robotics and Automation Letters},
  vol.~4, no.~2, pp. i--vii, 2018.

\bibitem{ladd2004measure}
A.~M. Ladd and L.~E. Kavraki, ``Measure theoretic analysis of probabilistic
  path planning,'' \emph{IEEE Transactions on Robotics and Automation},
  vol.~20, no.~2, pp. 229--242, 2004.

\bibitem{karaman2011sampling}
S.~Karaman and E.~Frazzoli, ``Sampling-based algorithms for optimal motion
  planning,'' \emph{The international journal of robotics research}, vol.~30,
  no.~7, pp. 846--894, 2011.

\bibitem{solovey2020revisiting}
K.~Solovey, L.~Janson, E.~Schmerling, E.~Frazzoli, and M.~Pavone, ``Revisiting
  the asymptotic optimality of {RRT},'' in \emph{2020 IEEE international
  conference on robotics and automation (ICRA)}.\hskip 1em plus 0.5em minus
  0.4em\relax IEEE, 2020, pp. 2189--2195.

\bibitem{edelsbrunner1983shape}
H.~Edelsbrunner, D.~Kirkpatrick, and R.~Seidel, ``On the shape of a set of
  points in the plane,'' \emph{IEEE Transactions on information theory},
  vol.~29, no.~4, pp. 551--559, 1983.

\bibitem{tedrake2019drake}
\BIBentryALTinterwordspacing
R.~Tedrake and the Drake Development~Team, ``Drake: Model-based design and
  verification for robotics,'' 2019. [Online]. Available:
  \url{https://drake.mit.edu}
\BIBentrySTDinterwordspacing

\bibitem{kuffner2000rrt}
J.~J. Kuffner and S.~M. LaValle, ``{RRT}-connect: An efficient approach to
  single-query path planning,'' in \emph{Proceedings 2000 ICRA. Millennium
  conference. IEEE international conference on robotics and automation.
  Symposia proceedings (Cat. No. 00CH37065)}, vol.~2.\hskip 1em plus 0.5em
  minus 0.4em\relax IEEE, 2000, pp. 995--1001.

\bibitem{graesdal2024towards}
B.~P. Graesdal, S.~Y. {Chew Chia}, T.~Marcucci, S.~Morozov, A.~Amice,
  P.~Parrilo, and R.~Tedrake, ``{Towards Tight Convex Relaxations for
  Contact-Rich Manipulation},'' in \emph{Proceedings of Robotics: Science and
  Systems}, Delft, Netherlands, July 2024.

\bibitem{chia2024gcs}
S.~Y. Chew~Chia, R.~H. Jiang, B.~P. Graesdal, L.~P. Kaelbling, and R.~Tedrake,
  ``{GCS*}: Forward heuristic search on implicit graphs of convex sets,''
  \emph{arXiv preprint arXiv:2407.08848}, 2024.

\end{thebibliography}

\appendices
\renewcommand{\thesubsectiondis}{\Alph{section}.\arabic{subsection}}
\renewcommand{\thesubsection}{\Alph{section}.\arabic{subsection}}

{
\crefalias{section}{appendix}
\crefalias{subsection}{appendix}

\appendix{}
\label{appx:proofs}

\begin{proof}[Proof of \Cref{lem:volume}]
    The action of $G$ on $\mc Q$ is properly discontinuous, so $\forall q\in\mc Q$, there is an open neighborhood $U_q$ of $q$ such that $\forall g\in G\cut\set{e}$, $g(U_q)\cap U_q=\emptyset$.
    $\set{\pi(U_q):q\in\mc Q}$ is an open cover of $\mc Q/G$, so by compactness, there is a finite subcover $\set{\pi(U_{q_i})}_{i=1}^n$.
    Now, define
    \begin{equation*}
        V_i=\pi\paren{U_{q_i}}\cut\bigcup_{j=1}^{i-1}\pi\paren{U_{q_j}}.
    \end{equation*}
    By construction $V_i\cap V_j=\emptyset$ for $i\ne j$, and $\bigcup_{i=1}^nV_i=\mc Q/G$, so $\bigcup_{i=1}^n\pi\inv(V_i)=\mc Q$.
    The preimage of each $V_i$ will be $\card{G}$ disjoint copies, and $\pi$ is a local isometry, so $\Vol(\pi\inv(V_i))=\card{G}\Vol(V_i)$.
    We conclude that
    \begin{multline*}
        \Vol(\mc Q/G)=\sum_{i=1}^{n}\Vol(V_i)=\frac{1}{\card{G}}\sum_{i=1}^{n}\Vol(\pi\inv(V_i))\\
        =\frac{1}{\card{G}}\Vol(\mc Q).\;\tag*{\qedhere}
    \end{multline*}
\end{proof}

\begin{proof}[Proof of \Cref{cor:free_volume}]
    We follow the proof of \Cref{lem:volume} above, except we intersect each $V_i$ with $\mc Q_\free/G$.
\end{proof}

\begin{proof}[A note on the proof of \Cref{lem:balls_to_balls}]
    Lemma 1.32 of \cite{cheeger1975comparison} states the result for an arbitrarily-small open neighborhood, but the proof uses a normal coordinate system, which will be bijective on any ball within the injectivity radius.
\end{proof}

\subsection*{Proof of \texorpdfstring{\Cref{thm:injectivity_radius}}{Theorem 1}}
\label{proof:injectivity_radius}

The proof of the bound on the injectivity radius is more involved.
First, we show that points are evenly-spaced along one-parameter subgroups of $\mc Q$. Fix $x\in\mc Q$ and $g\in G\cut\set{e}$.
\begin{lemma}
    $\rmd_\mc Q(x,g\cdot x)=\rmd_\mc Q(g\cdot x, g^2\cdot x)$.
\end{lemma}
\begin{proof}
    Let $\gamma\hspace{-0.2em}:\hspace{-0.2em}[0,1]\hspace{-0.2em}\to\hspace{-0.2em}\mc Q$ be a minimizing geodesic connecting $x$ to $g\cdot x$.
    Because $G$ acts on $\mc Q$ by isometries, $L_g\of\gamma$ is also a geodesic, of the same length, connecting $g\cdot x$ to $g\cdot(g\cdot x)=g^2\cdot x$.
    Clearly, this must also be minimzing; if there is a shorter geodesic $\tilde\gamma$ connecting $g\cdot x$ to $g^2\cdot x$, then $L_{g\inv}\of\tilde\gamma$ would be a geodesic of the same length connecting $g$ to $g\cdot x$, violating our assumption that $\gamma$ is minimizing.
    Since $L_g\of\gamma$ is minimizing, $\rmd_\mc Q(g\cdot x,g^2\cdot x)=\rmd_\mc Q(x,g\cdot x)$.
\end{proof}
\begin{lemma}
    $\forall x\in\mc Q$, $\forall g,h\in G$ distinct, $\rmd_\mc Q(g\cdot x,h\cdot x)\ge 2r_\mrm{inj}(\mc Q)/\card{G}$.
    \label{lem:same_orbit}
\end{lemma}
\begin{proof}
    Without loss of generality, assume $h=e$.
    Consider the one parameter subgroup generated by $g$, and suppose $d(x,g\cdot x)<2r_\mrm{inj}(\mc Q)/\card{G}$.
    Consider the geodesic $\gamma:t\mapsto\exp_x(t\log_x(g))$.
    We know $\ms L(\restr{\gamma}_{[0,1]})<2r_\mrm{inj}(\mc Q)/\card{G}$, so
    \begin{multline*}
        \ms L(\restr{\gamma}_{[0,\card{G}]})=\sum_{i=1}^{\card{G}}\ms L(\restr{\gamma}_{[i-1,i]})\\
        =\card{G}\ms L(\restr{\gamma}_{[0,1]})<2r_\mrm{inj}(\mc Q).
    \end{multline*}
    But $g^{\card{G}}=e$ by Lagrange's Theorem, so we have constructed a geodesic loop.
    The length of a geodesic loop in $\mc Q$ must be at least $2r_\mrm{inj}(\mc Q)$, but $\restr{\gamma}_{[0,\card{G}]}$ is shorter.
    This is a contradiction.
\end{proof}
\begin{proof}[Proof of \Cref{thm:injectivity_radius}]
    Suppose $r_\mrm{inj}(\mc Q/G)<r_\mrm{inj}(\mc Q)/\card{G}$.
    Then there must exist distinct $[x],[y]\in\mc Q/G$ and distinct geodesics $\gamma,\xi$ connecting $[x]$ to $[y]$, both of length strictly less than $r_\mrm{inj}(\mc Q)/\card{G}$.
    (If no such points and geodesics existed, then the injectivity radius would be at least $r_\mrm{inj}(\mc Q)/\card{G}$.)
    We can lift $\gamma$ and $\xi$ to geodesics $\tilde\gamma,\tilde\xi$ of $\mc Q$, such that $\tilde\gamma(0)=\tilde\xi(0)=x\in[x]$, $\tilde\gamma(1)=y\in[y]$, $\tilde\xi(1)=y'\in[y]$, $\pi\of\tilde\gamma=\gamma$, and $\pi\of\tilde\xi=\xi$.
    (We leverage transitivity of the group action to ensure $\tilde\gamma$ and $\tilde\xi$ start at the same point in $\mc Q$ when lifted.)
    We clearly cannot have $y=y'$, or else we would violate the injectivity radius of $\mc Q$.
    Because $y$ and $y'$ are in the same orbit, \Cref{lem:same_orbit} requires $\rmd_\mc Q(y,y')\ge 2r_\mrm{inj}(\mc Q)/\card{G}$.
    But the triangle inequality requires that
    \begin{equation}
        \rmd_\mc Q(y,y')\le\rmd_\mc Q(x,y)+\rmd_\mc Q(x,y')<2r_\mrm{inj}(\mc Q)/\card{G}.
        \label{eq:injectivity_triangle_inequality}
    \end{equation}
    We conclude that $r_\mrm{inj}(\mc Q/G)\ge r_\mrm{inj}(\mc Q)/\card{G}$.
\end{proof}

}

\end{document}